\documentclass{article}

\usepackage{microtype}
\usepackage{graphicx}
\usepackage{booktabs} %
\usepackage{dsfont}

\usepackage{hyperref}

\usepackage{amsmath,amsfonts,bm,amssymb,mathtools,amsthm}
\usepackage{color,xcolor}
\usepackage{booktabs}
\usepackage{amsmath,environ}
\usepackage{subcaption}
\usepackage{multirow,multicol}
\usepackage{thmtools,thm-restate}

\newtheorem{theorem}{Theorem}

\newtheorem{proposition}{Proposition}
\newtheorem{corollary}{Corollary}
\newtheorem{lemma}{Lemma}

\def\eqref#1{Eq.(\ref{#1})}

\def\1{\bm{1}}

\def\norm#1{\lVert #1 \rVert}
\newcommand*\diff{\mathop{}\!\mathrm{d}}

\def\vx{{\bm{x}}}

\DeclareMathAlphabet{\mathsfit}{\encodingdefault}{\sfdefault}{m}{sl}
\SetMathAlphabet{\mathsfit}{bold}{\encodingdefault}{\sfdefault}{bx}{n}

\def\gF{{\mathcal{F}}}

\def\gL{{\mathcal{L}}}

\def\gP{{\mathcal{P}}}

\def\gR{{\mathcal{R}}}

\def\gX{{\mathcal{X}}}

\newcommand{\iftpq}{I_{f}(T; P, Q)}
\newcommand{\iwtrpq}{I_{W}(T; P, Q)}
\newcommand{\ltrpq}{\ell_f(T, r; P, Q)}
\newcommand{\LRtpq}{\gL^{\gR}_f(T;P,Q)}

\newcommand{\ipm}{\mathrm{IPM}}
\newcommand{\pdata}{P_{\rm{data}}}

\newcommand{\E}{\mathbb{E}}

\newcommand{\R}{\mathbb{R}}

\newcommand{\Var}{\mathrm{Var}}

\DeclareMathOperator*{\argmin}{arg\,min}

\newcommand{\bb}[1]{{\mathbb{#1}}}

\usepackage[accepted]{icml2020}

\makeatletter
\usepackage{xspace}
\def\@onedot{\ifx\@let@token.\else.\null\fi\xspace}
\DeclareRobustCommand\onedot{\futurelet\@let@token\@onedot}

\def\iid{i.i.d\onedot}

\icmltitlerunning{}

\begin{document}

\twocolumn[
\icmltitle{Bridging the Gap Between $f$-GANs and Wasserstein GANs}

\icmlsetsymbol{equal}{*}

\begin{icmlauthorlist}
\icmlauthor{Jiaming Song}{to}
\icmlauthor{Stefano Ermon}{to}
\end{icmlauthorlist}

\icmlaffiliation{to}{Stanford University}

\icmlcorrespondingauthor{Jimaing Song}{tsong@cs.stanford.edu}

\icmlkeywords{generative adversarial networks, f-gans, wasserstein gans}

\vskip 0.3in
]

\printAffiliationsAndNotice{}  %

\begin{abstract}
    Generative adversarial networks (GANs) variants approximately minimize divergences between the model and the data distribution using a discriminator. Wasserstein GANs (WGANs) enjoy superior empirical performance, however, unlike in $f$-GANs, the discriminator does not provide an estimate for the ratio between model and data densities, which is useful in applications such as inverse reinforcement learning. To overcome this limitation, we propose an new training objective where we additionally optimize over a set of importance weights over the generated samples. By suitably constraining the feasible set of importance weights, we obtain a family of objectives which includes and generalizes the original $f$-GAN and WGAN objectives. We show that a natural extension outperforms WGANs while providing density ratios as in $f$-GAN, and demonstrate empirical success on distribution modeling, density ratio estimation and image generation. %

\end{abstract}

\section{Introduction}
Learning generative models to sample from complex, high-dimensional distributions is an important task in machine learning with many important applications, such as image generation~\citep{kingma2013auto}, imitation learning~\citep{ho2016generative} and representation learning~\citep{chen2016infogan}. 
Generative adversarial networks (GANs,~\citet{goodfellow2014generative}) are likelihood-free deep generative models~\citep{mohamed2016learning} based on finding the equilibrium of a two-player minimax game between a generator and a critic (discriminator). Assuming the optimal critic is obtained, one can cast the GAN learning procedure as minimizing a discrepancy measure between the distribution induced by the generator and the training data distribution. 

Various GAN learning procedures have been proposed for different discrepancy measures. 
$f$-GANs~\citep{nowozin2016f} minimize a variational approximation of the $f$-divergence between two distributions~\citep{csiszar1964informationstheoretische,nguyen2008estimating}. In this case, the critic acts as a density ratio estimator~\citep{uehara2016generative,grover2017boosted}, i.e., 
it estimates if points are more likely to be generated by the data or the generator distribution.
This includes the original GAN approach~\citep{goodfellow2014generative} which can be seen as minimizing a variational approximation to the Jensen-Shannon divergence. Knowledge of the density ratio between two distributions can be used for importance sampling and in a range of practical applications such as mutual information estimation~\citep{devon2018learning}, off-policy policy evaluation~\citep{liu2018breaking}, and de-biasing of generative models~\citep{grover2019bias}. %

Another family of GAN approaches are developed based on Integral Probability Metrics (IPMs,~\citet{muller1997integral}), where the critic (discriminator) is restricted to particular function families. For the family of Lipschitz-$1$ functions, the IPM reduces to the Wasserstein-1 or earth mover's distance~\citep{rubner2000earth}, which motivates the Wasserstein GAN (WGAN,~\citet{arjovsky2017wasserstein}) setting. Various approaches have been applied to enforce Lipschitzness, including weight clipping~\citep{arjovsky2017wasserstein}, gradient penalty~\citep{gulrajani2017improved} and spectral normalization~\citep{miyato2018spectral}. Despite its strong empirical success in image generation~\citep{karras2017progressive,brock2018large}, 
the learned critic cannot be interpreted as a density ratio estimator, which limits its usefulness for importance sampling or other GAN-related applications such as inverse reinforcement learning~\citep{yu2019multi}.  %

In this paper, we address this problem via a generalized view of $f$-GANs and WGANs. The generalized view introduces importance weights over the generated samples in the critic objective, allowing prioritization over the training of different samples. The algorithm designer can select suitable feasible sets to constrain the importance weights; we show that both $f$-GAN and WGAN are special cases to this generalization when specific feasible sets are considered. We further discuss cases that select alternative feasible sets where divergences other than $f$-divergence and IPMs can be obtained. 

To derive concrete algorithms, we turn to a case where the importance weights belong to the set of valid density ratios over the generated distribution. In certain cases, the optimal importance weights can be obtained via closed-form solutions, bypassing the need to perform an additional inner-loop optimization.
We discuss one such approach, named KL-Wasserstein GAN (KL-WGAN), that is easy to implement from existing WGAN approaches, and is compatible with state-of-the-art GAN architectures. We evaluate KL-WGAN empirically on distribution modeling, density estimation and image generation tasks. Empirical results demonstrate that KL-WGAN enjoys superior quantitative performance compared to its WGAN counterparts on several benchmarks. %

\section{Preliminaries}

\paragraph{Notations} Let $X$ denote a random variable with separable sample space $\gX$ and let $\gP(\gX)$ denote the set of all probability measures over the Borel $\sigma$-algebra on $\gX$. We use $P$, $Q$ to denote probabiliy measures, and $P \ll Q$ to denote $P$ is absolutely continuous with respect to $Q$, i.e. the Radon-Nikodym derivative $\diff P / \diff Q$ exists.
Under $Q \in \gP(\gX)$, the $p$-norm of a function $r: \gX \to \R$ is defined as 
\begin{align}
    \Vert r \Vert_p := \left(\int |r(\vx)|^p \mathrm{d} Q(\vx)\right)^{1/p},
\end{align} with $\norm{r}_\infty = \lim_{p \to \infty} \norm{r}_p$. 
The set of locally $p$-integrable functions is defined as 
\begin{align}
    L^p(Q) := \{ r: \gX \to \bb{R} : \norm{r}_p < \infty \},
\end{align}
i.e. its norm with respect to $Q$ is finite. We denote $L^p_{\geq 0}(Q) := \{r \in L^p(Q) : \forall \vx \in \gX, r(\vx) \geq 0 \}$
which considers non-negative functions in $L^p(Q)$. The space of probability measures wrt. $Q$ is defined as 
\begin{align}
    \Delta(Q) := \{ r \in L^1_{\geq 0}(Q) : \Vert r \Vert_1 = 1 \}. \label{eq:delta-q-def}
\end{align}
For example, for any $P \ll Q$, $\diff P / \diff Q \in \Delta(Q)$ because $\int (\diff P / \diff Q) \diff Q = 1$. 
We define $\mathds{1}$ such that $\forall \vx \in \gX$, $\mathds{1}(\vx) = 1$, and define $\mathrm{im}(\cdot)$ and $\mathrm{dom}(\cdot)$ as image and domain of a function respectively.

\paragraph{Fenchel duality} 

For functions $g: \gX \to \R$ defined over a Banach space $\gX$, the Fenchel dual of $g$, $g^{*}: \gX^{*} \to \R$ is defined over the dual space $\gX^{*}$ by:
\begin{align}
    g^{*}(\vx^{*}) := \sup_{\vx \in \gX} \langle \vx^{*}, \vx \rangle - g(\vx),
\end{align}
where $\langle \cdot, \cdot \rangle$ is the duality paring. For example, the dual space of 
$\R^d$ is also $\R^d$ and $\langle \cdot, \cdot \rangle$ is the usual inner product~\citep{rockafellar1970convex}.

\paragraph{Generative adversarial networks} In generative adversarial networks (GANs,~\citet{goodfellow2014generative}), the goal is to fit an (empirical) data distribution $P_{\mathrm{data}}$ with an implicit generative model over $\gX$, denoted as $Q_\theta \in \gP(\gX)$. $Q_\theta$ is defined implicitly via the process $X = G_\theta(Z)$, where $Z$ is a random variable with a fixed prior distribution. %
Assuming access to $\iid$ samples from $P_{\mathrm{data}}$ and $Q_\theta$, a discriminator $T_\phi: \gX \to [0, 1]$ is used to classify samples from the two distributions, leading to the following objective:
\begin{align}
    \min_\theta \max_{\phi} \bb{E}_{\vx \sim P_{\mathrm{data}}}[\log T_\phi(\vx)] + \bb{E}_{\vx \sim Q_\theta}[\log (1 - T_\phi(\vx))]. \nonumber
\end{align}
If we have infinite samples from $P_{\mathrm{data}}$, and $T_\phi$ and $Q_\theta$ are 
sufficiently 
expressive, %
then the above minimax objective will reach an equilibrium where $Q_\theta = P_{\mathrm{data}}$ and $T_\phi(\vx) = 1/2$ for all $\vx \in \gX$.

\subsection{Variational Representation of $f$-Divergences}
For any convex and semi-continuous function $f: [0, \infty) \to \R$ satisfying $f(1) = 0$, the $f$-divergence~\citep{csiszar1964informationstheoretische,ali1966general} between two probabilistic measures $P, Q \in \gP(\gX)$ is defined as:
\begin{align}
    D_{f}(P \Vert Q) & := \bb{E}_Q\left[f\left(\frac{\diff P}{\diff Q}\right)\right] \\
    \label{eq:fdiv}
    & = \int_\gX f\left(\frac{\diff P}{\diff Q}(\vx)\right) \diff Q(\vx), 
\end{align}
if $P \ll Q$ and $+\infty$ otherwise. \citet{nguyen2010estimating} derive a general variational method to estimate $f$-divergences given only samples from $P$ and $Q$.

\begin{lemma}[\citet{nguyen2010estimating}]
\label{thm:nwj}
$\forall P, Q \in \gP(\gX)$ such that $P \ll Q$, and differentiable $f$:
\begin{gather}
    D_{f}(P \Vert Q) = \sup_{T \in L^\infty(Q)} \iftpq,
    \label{eq:nwj} \\
    \text{where} \quad \iftpq := \bb{E}_P[T(\vx)] - \bb{E}_Q[f^{*}(T(\vx))] \label{eq:iftpq}
\end{gather}
and the supremum is achieved when 
$T = f'(\diff P / \diff Q) %
$.
\end{lemma}

In the context of GANs, \citet{nowozin2016f} proposed variational $f$-divergence minimization where one estimates $D_f(\pdata \Vert Q_\theta)$ with the variational lower bound in \eqref{eq:nwj} while minimizing over $\theta$ the estimated divergence. This leads to the $f$-GAN objective:
\begin{align}
    \min_\theta \max_{\phi} \bb{E}_{\vx \sim P_{\mathrm{data}}}[T_\phi(\vx)] - \bb{E}_{\vx \sim Q_\theta}[f^{*} (T_\phi(\vx))],
\end{align}
where the original GAN objective is a special case for $f(u) = u \log u - (u+1) \log (u+1) + 2\log 2$.

\subsection{Integral Probability Metrics and Wasserstein GANs}

For a fixed class of real-valued bounded Borel measurable
functions $\gF$ on $\gX$, the integral probability metric (IPM) based on $\gF$ and between %
$P, Q \in \gP(\gX)$ is defined as:
\begin{align}
    \ipm_{\gF}(P, Q) := \sup_{T \in \gF} \left\vert \int T(\vx) \diff P(\vx) - \int T(\vx) \diff Q(\vx) \right\vert. \nonumber %
\end{align}
If for all $T \in \gF$, $-T \in \gF$ then $\ipm_{\gF}$ forms a metric over $\gP(\gX)$~\citep{muller1997integral}; we assume this is always true for $\gF$ in this paper (so we can remove the absolute values). 
In particular, if $\gF$ is the set of all bounded $1$-Lipschitz functions %
with respect to the metric over $\gX$, then the corresponding IPM becomes the Wasserstein distance between $P$ and $Q$~\citep{villani2008optimal}. This motivates the Wasserstein GAN objective~\citep{arjovsky2017wasserstein}:
\begin{align}
    \min_\theta \max_\phi \bb{E}_{\vx \sim P_{\mathrm{data}}}[T_\phi(\vx)] - \bb{E}_{\vx \sim Q_\theta}[T_\phi(\vx)], \label{eq:wgan}
\end{align}
where $T_\phi$ is regularized to be approximately $k$-Lipschitz for some $k$. Various approaches have been applied to enforce Lipschitzness of neural networks, including weight clipping~\citep{arjovsky2017wasserstein}, gradient penalty~\citep{gulrajani2017improved}, and spetral normalization over the weights~\citep{miyato2018spectral}.

Despite its strong empirical performance, WGAN has two drawbacks. First, unlike $f$-GAN (Lemma \ref{thm:nwj}), it does not naturally recover a density ratio estimator from the critic. Granted, the WGAN objective corresponds to an $f$-GAN one~\citep{sriperumbudur2009on} when $f(x) = 0$ if $x = 1$ and $f(x) = +\infty$ otherwise,
so that $f^{*}(x) = x$; however, we can no longer use Lemma~\ref{thm:nwj} to recover density ratios given an optimal critic $T$, because the 
derivative $f'(x)$
does not exist. %
Second, WGAN places the same weight on the objective for each generated sample, which could be sub-optimal when the generated samples are of different qualities. %

\section{A Generalization of $f$-GANs and WGANs}

In order to achieve the best of both worlds, we propose an alternative generalization to the critic objectives to \emph{both} $f$-GANs and WGANs. Consider the following functional: %
\begin{align}
   & \ltrpq \\
   := \ & \E_{\vx \sim Q}[f(r(\vx))] + \E_{\vx \sim P}[T(\vx)] - \E_{\vx \sim Q}[r(\vx) \cdot T(\vx)] \nonumber
\end{align}
which depends on the distributions $P$ and $Q$, the critic function $T: \gX \to \R$, and an additional function $r: \gX \to \R$. %
For conciseness, we remove the dependency on the argument $\vx$ for $T, r, P, Q$ in the remainder of the paper. 

The function $r: \gX \to \R$ here plays the role of ``importance weights'', as they changes the weights to the critic objective over the generator samples. When $r = \diff P / \diff Q$, the objective above simplifies to $\bb{E}_{Q}[f(\diff P /\diff Q)]$ 
which is exactly the definition of the $f$-divergence between $P$ and $Q$ (Eq. \ref{eq:fdiv}). 

To recover an objective over only the critic $T$, we minimize $\ell_f$ as a function of $r$ over a suitable set $\gR \subseteq L_{\geq 0}^\infty(Q)$, thus eliminating the dependence over $r$:
\begin{align}
    & \LRtpq := \inf_{r \in \gR} \ltrpq
    \label{eq:generalopt} %
\end{align}

We note that the minimization step is performed within a particular set $\gR \subseteq L^\infty(Q)$, which can be selected by the algorithm designer. %
The choice of the set $\gR$ naturally gives rise to different critic objectives.
As we demonstrate below (and in Figure~\ref{fig:fwgan-relation}), we can obtain critic objectives for $f$-GAN as well as WGANs as special cases via different choices of $\gR$ in $\LRtpq$. %

\subsection{Recovering the $f$-GAN Critic Objective} First, we can recover the critic in the $f$-GAN objective by setting $\gR = L^\infty_{\geq 0}(Q)$, which is the set of all non-negative functions in $L^\infty(Q)$. Recall from Lemma~\ref{thm:nwj} the $f$-GAN objective:
\begin{align}
    D_f(P \Vert Q) = \sup_{T \in L^\infty(Q)} \iftpq
\end{align}
where $\iftpq := \bb{E}_P[T] - \bb{E}_Q[f^{*}(T)]$ as defined in Lemma~\ref{thm:nwj}. The following proposition shows that when $\gR = L^\infty_{\geq 0}(Q)$, we recover $I_f = \gL_f^{\gR}$.
\begin{proposition}
\label{thm:lr-f}
Assume that $f$ is differentiable at $[0, \infty)$. $\forall P, Q \in \gP(\gX)$ such that $P \ll Q$, and $\forall T \in \gF \subseteq L^\infty(Q)$ such that $\mathrm{im}(T) \subseteq \mathrm{dom}((f')^{-1})$,
\begin{align}
    \iftpq = \inf_{r \in L^\infty_{\geq 0}(Q)} \ltrpq.
\end{align}
where $I_f(T; P, Q) := \bb{E}_P[T] - E_{Q}[f^{*}(T)]$.
\end{proposition}

\begin{proof}
From Fenchel's inequality we have for convex $f: \bb{R} \to \bb{R}$, $\forall T(\vx) \in \bb{R}$ and $\forall r(\vx) \geq 0$, 
$
    f(r(\vx)) + f^{*}(T(\vx)) \geq r(\vx) T(\vx)
$
where equality holds when
$T(\vx) = f'(r(\vx))$. Taking the expectation over $Q$, we have
\begin{align}
    \E_{Q}[f(r)] - \bb{E}_Q[rT] \geq  - \E_{Q}[f^{*}(T)]; \label{eq:eq-dual}
\end{align}
applying this to the definition of $\ltrpq$, we have:
\begin{align}
     & \  \ltrpq := \E_{Q}[f(r)] + \bb{E}_P[T] - \bb{E}_{Q}[rT] \nonumber \\
  \geq & \  \bb{E}_P[T] - \bb{E}_Q[f^{*}(T)] = \iftpq. \label{eq:nwj-2}
\end{align}
where the inequality comes from Equation~\ref{eq:eq-dual}.
The inequality becomes an equality when $r(\vx) = (f')^{-1}(T(\vx))$ for all $\vx \in \gX$. We note that such a case can be achieved, i.e., $(f')^{-1}(T) \in L_{\geq 0}^\infty(Q)$,  because $\forall \vx \in \gX, (f')^{-1}(T(\vx)) \in \mathrm{dom}(f) = [0, \infty)$ from the assumption over $\mathrm{im}(T)$. Therefore, taking the infimum over $r \in L^\infty_{\geq 0}(Q)$, we have:
\begin{align}
    \iftpq = \inf_{r \in L^\infty_{\geq 0}(Q)} \ltrpq,
\end{align}
which completes the proof.
\end{proof}

\subsection{Recovering the WGAN Critic Objective} Next, we recover the WGAN critic objective (IPM) by setting $\gR = \{\mathds{1}\}$, where $\mathds{1}(x) = 1$ is a constant function.
First, we can equivalently rewrite the definition of an IPM using the following notation:
\begin{align}
    \ipm_\gF(P, Q) = \sup_{T \in \gF} I_W(T; P, Q)
\end{align}
where $I_W$ represents the critic objective.
We show that $I_W = \gL_f^{\gR}$ when $\gR = \{\mathds{1}\}$ as follows.
\begin{proposition}
\label{thm:lr-w}
$\forall P, Q \in \gP(\gX)$ such that $P \ll Q$, and $\forall T \in \gF \subseteq L^\infty(Q)$:
\begin{align}
    I_W(T; P, Q) = \inf_{r \in \{\mathds{1}\}} \ltrpq
\end{align}
where $\iwtrpq := \bb{E}_P[T] - \bb{E}_{Q}[T]$.
\end{proposition}

\begin{proof} As $\{\mathds{1}\}$ has only one element, the infimum is:
\begin{align}
\ell_f(T, \mathds{1}; P, Q) &= \bb{E}_Q[f(1)] + \E_P[T] - \E_Q[T] \\
& = I_W(T; P, Q)
\end{align}
where we used $f(1) = 0$ for the second equality.
\end{proof}

The above propositions show that $\gL_f^\gR$ generalizes both $f$-GAN and WGANs critic objectives by setting $\gR = L^\infty_{\geq 0}(Q)$ and $\gR = \{\mathds{1}\}$ respectively.

\begin{figure}
    \centering
    \includegraphics[width=0.49\textwidth]{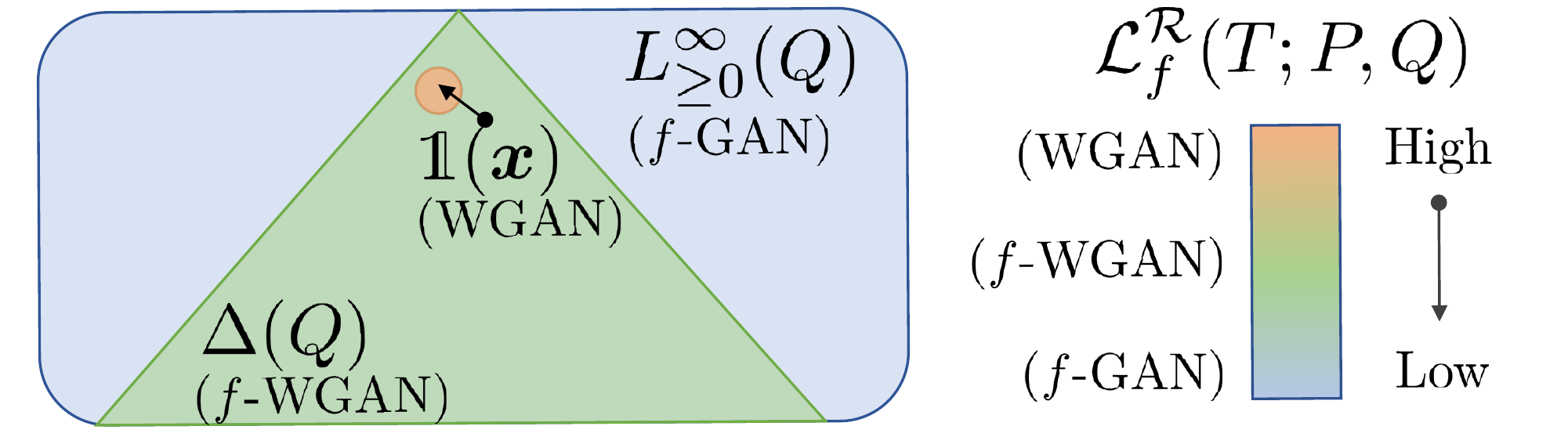}
    \caption{
    (Left) Minimization over different $\gR$ in $\gL_f^\gR$ gives different critic objectives. Minimizing over $L^\infty_{\geq 0}(Q)$ recovers $f$-GAN (blue set), minimizing over $\{\mathds{1}\}$ recovers WGAN (orange set), and minimizing over $\Delta(Q)$ recovers $f$-WGAN (green set). (Right) Naturally, as we consider smaller sets $\gR$ to minimize over, the critic objective becomes larger for the same $T$.}
    \label{fig:fwgan-relation}
\end{figure}

\subsection{Extensions to Alternative Constraints}

The generalization with $\gL_f^{\gR}$ allows us to introduce new objectives when we consider alternative choices for the constraint set $\gR$. We consider sets $\gR$ such that $\{\mathds{1}\} \subseteq \gR \subseteq L^\infty_{\geq 0}(Q)$. The following proposition shows for some fixed $T$, the corresponding objective with $\gR$ is bounded between the $f$-GAN objective (where $\gR = L^\infty_{\geq 0}(Q)$) and the WGAN objective (where $\gR = \{\mathds{1}\}$).
\begin{restatable}{proposition}{lfrbounded}
\label{thm:ub-f}
$\forall P, Q \in \gP(\gX)$ such that $P \ll Q$, $\forall T \in L^\infty(Q)$ such that $\mathrm{im}(T) \subseteq \mathrm{dom}((f')^{-1})$, and $\forall \gR \subseteq L_{\geq 0}^\infty(Q)$ such that $\{\mathds{1}\} \subseteq \gR$ we have:
\begin{align}
   \iftpq  \leq \LRtpq \leq \iwtrpq.
\end{align}
\end{restatable}
\begin{proof}
In Appendix~\ref{app:proofs}.
\end{proof}
We visualize this in Figure~\ref{fig:fwgan-relation}. Selecting the set $\gR$ allows us to  control the critic objective in a more flexible manner, interpolating between the $f$-GAN critic and the IPM critic objective and finding suitable trade-offs.
Moreover, if we additionally take the supremum of $\LRtpq$ over $T$, the result will be bounded between the supremum of $I_f$ over $T$ (corresponding to the $f$-divergence) and the supremum of $I_W$ over $T$, as stated in the following theorem.
\begin{restatable}{theorem}{div}
\label{thm:div}
For $\{\mathds{1}\} \subseteq \gR \subseteq L_{\geq 0}^\infty(Q)$, define
\begin{align}
    D_{f, \gR}(P \Vert Q) := \sup_{T \in \gF} \LRtpq
\end{align}
where $\gF := \{T: \gX \to \mathrm{dom}((f')^{-1}), T \in L^\infty(Q)\}$. Then
\begin{align}
    D_{f}(P \Vert Q) \leq D_{f, \gR}(P \Vert Q) \leq \sup_{T \in \gF} I_W(T; P, Q).
\end{align}
\end{restatable}
\begin{proof}
In Appendix~\ref{app:proofs}. %
\end{proof}

A natural corollary is that $D_{f, \gR}$ defines a divergence between two distributions.
\begin{corollary}
\label{thm:divergence}
    $D_{f, \gR}(P \Vert Q)$ defines a divergence between $P$ and $Q$: $D_{f, \gR}(P \Vert Q) \geq 0$ for all $P, Q \in \gP(\gX)$, and $D_{f, \gR}(P \Vert Q) = 0$ if and only if $P = Q$.
\end{corollary}
This allows us to interpret the corresponding GAN algorithm as variational minimization of a certain divergence bounded between the corresponding $f$-divergence and IPM.

\section{Practical $f$-Wasserstein GANs}

As a concrete example, we consider the set $\gR = \Delta(Q)$, which is the set of all valid density ratios over $Q$. We note that $\{\mathds{1}\} \subset \Delta(Q) \subset L^\infty_{\geq 0}(Q)$ (see Figure~\ref{fig:fwgan-relation}), so the corresponding objective is a divergence (from Corollary~\ref{thm:divergence}).
We can then consider the variational divergence minimization objective over $\gL_f^{\Delta(Q)}(T; P, Q)$:
\begin{align}
   \inf_{Q \in \gP(\gX)} \sup_{T \in \gF} \inf_{r \in \Delta(Q)} \ltrpq, \label{eq:fwgan}
\end{align}
We name this the ``$f$-Wasserstein GAN'' ($f$-WGAN) objective, since it provides an interpolation between $f$-GAN and Wasserstein GANs while recovering a density ratio estimate between two distributions.

\subsection{KL-Wasserstein GANs}
For the $f$-WGAN objective in \eqref{eq:fwgan}, the trivial algorithm would have to perform iterative updates to three quantities $Q$, $T$ and $r$, which involves three 
nested optimizations. While this seems impractical,
we show that for certain choices of $f$-divergences, we can obtain closed-form solutions for the optimal $r \in \Delta(Q)$ in the innermost minimization; this bypasses the need to perform an inner-loop optimization over $r \in \Delta(Q)$, as we can simply assign the optimal solution from the close-form expression.

\begin{restatable}{theorem}{klgan}\label{thm:kl-gan}
Let $f(u) = u \log u$ and $\gF$ a set of real-valued bounded measurable functions on $\gX$. For any fixed choice of $P, Q$, and $T \in \gF$, we have
\begin{gather}
    \argmin_{r \in \Delta(Q)} \E_{Q}[f(r)] + \bb{E}_P[T] - \bb{E}_{Q}[r \cdot T] 
    = \frac{e^T}{\bb{E}_Q[e^T]} \label{eq:kl-gan}
\end{gather}
\end{restatable}
\begin{proof}
In Appendix~\ref{app:proofs}.
\end{proof}
The above theorem shows that if the $f$-divergence of interest is the KL divergence, we can directly obtain the optimal $r \in \Delta(Q)$ using \eqref{eq:kl-gan} for any fixed critic $T$.
Then, we can apply this $r$ to the $f$-WGAN objective, and perform gradient descent updates on $Q$ and $T$ only. %
Avoiding the optimization procedure over $r$ allows us to propose practical algorithms that are similar to existing WGAN procedures.
In Appendix~\ref{app:chi-square}, we show a similar argument with $\chi^2$-divergence, another $f$-divergence admitting a closed-form solution, and discuss its connections with the $\chi^2$-GAN approach~\citep{tao2018chi}.

\subsection{Implementation Details}

In Algorithm~\ref{alg:kl-wgan}, we describe KL-Wasserstein GAN (KL-WGAN), a practical algorithm motivated by the $f$-WGAN objectives based on the observations in Theorem~\ref{thm:kl-gan}. 
We note that $r_0$ corresponds to selecting the optimal value for $r$ from Theorem~\ref{thm:kl-gan}; once $r_0$ is selected, we ignore the effect of $E_Q[f(r_0)]$ to the objective and optimize the networks with the remaining terms, which corresponds to weighting the generated samples with $r_0$; the critic will be updated as if the generated samples are reweighted. In particular, $\nabla_\phi(D_0 - D_1)$ corresponds to the critic gradient ($T$, which is parameterized by $\phi$) and $\nabla_\theta D_1$ corresponds to the generator gradient ($Q$, parameterized by $\theta$). 

In terms of implementation, the only differences between KL-WGAN and WGAN are between lines \ref{alg:kl-gan-fake-start} and \ref{alg:kl-gan-fake-end}, where WGAN will assign $r_0(\vx) = 1$ for all $\vx \sim Q_m$. 
In contrast, KL-WGAN ``importance weights'' the samples using the critic, in the sense that it will assign higher weights to samples that have large $T_\phi(\vx)$ and lower weights to samples that have low $T_\phi(\vx)$. This will encourage the generator $Q_\theta(\vx)$ to put more emphasis on samples that have high critic scores. 
It is relatively easy to implement the KL-WGAN algorithm from an existing WGAN implementation, as we only need to modify the loss function. We present an implementation of KL-WGAN losses (in PyTorch) in Appendix~\ref{app:klwgan-impl}.

\begin{algorithm}[t]
\caption{Pseudo-code for KL-Wasserstein GAN}
\label{alg:kl-wgan}
\begin{algorithmic}[1]
\STATE \textbf{Input:} the (empirical) data distribution $\pdata$; %
\STATE \textbf{Output}: implicit generative model $Q_\theta$.

\STATE Initialize generator $Q_\theta$ and discriminator $T_\phi$.
\REPEAT 
\STATE Draw $P_m :=$ $m$ \iid samples from $\pdata$;
\STATE Draw $Q_m :=$ $m$ \iid samples from $Q_\theta(\vx)$.
\STATE Compute $D_1 := \E_{P_m}[T_\phi(\vx)]$ (real samples)
\FORALL{$\vx \in Q_m$  (fake samples)} \label{alg:kl-gan-fake-start}
\STATE Compute $r_0(\vx) := e^{T_\phi(\vx)} / \E_{Q_m}[e^{T_\phi(\vx)}]$ \label{alg:kl-gan-biased}
\ENDFOR
\STATE Compute $D_0 := \E_{Q_m}[r_0(\vx) T_\phi(\vx)]$. \label{alg:kl-gan-fake-end}
\STATE Perform SGD over $\theta$ with $-\nabla_\theta D_0$;
\STATE Perform SGD over $\phi$ with $\nabla_\phi (D_0 - D_1)$.
\STATE Regularize $T_\phi$ to satisfy $k$-Lipschitzness.
\UNTIL{Stopping criterion}
\RETURN{learned implicit generative model $Q_\theta$.}

\end{algorithmic}
\end{algorithm}

While the mini-batch estimation for $r_0(\vx)$ provides a biased estimate to the optimal $r \in \Delta(Q)$ (which according to Theorem \ref{thm:kl-gan} is $e^{T_\theta(\vx)} / \E_Q[e^{T_\theta(\vx)}]$, i.e., normalized with respect to $Q$ instead of over a minibatch of $m$ samples as done in line \ref{alg:kl-gan-fake-start}), we found that this does not affect performance significantly. %
We further note that computing $r_0(\vx)$ does not require additional network evaluations, so the computational cost for each iteration is nearly identical between WGAN and KL-WGAN. To promote reproducible research, we include code in the supplementary material.

\section{Related Work}

\subsection{$f$-divergences, IPMs and GANs} 
Variational $f$-divergence minimization and IPM minimization paradigms are widely adopted in GANs. A non-exhaustive list includes $f$-GAN~\citep{nowozin2016f}, Wasserstein GAN~\citep{arjovsky2017wasserstein}, MMD-GAN~\citep{li2017mmd}, WGAN-GP~\citep{gulrajani2017improved}, SNGAN~\citep{miyato2018spectral}, LSGAN~\citep{mao2017least}, etc. The $f$-divergence paradigms enjoy better interpretations over the role of learned discriminator (in terms of density ratio estimation), whereas IPM-based paradigms enjoy better training stability and empirical performance. 
Prior work have connected IPMs with $\chi^2$ divergences between mixtures of data and model distributions~\citep{mao2017least,tao2018chi,mroueh2017fisher}; our approach can be applied to $\chi^2$ divergences as well, and we discuss its connections with $\chi^2$-GAN in Appendix~\ref{app:chi-square}.

Several works~\citep{liu2017approximation,farnia2018a} considered restricting function classes directly over the $f$-GAN objective; \citet{husain2019adversarial} show that restricted $f$-GAN objectives are lower bounds to Wasserstein autoencoder~\citep{tolstikhin2017wasserstein} objectives, aligning with our argument for $f$-GAN and WGAN (Figure~\ref{fig:fwgan-relation}).

Our approach is most related to regularized variational $f$-divergence estimators~\citep{nguyen2010estimating,ruderman2012tighter} and linear $f$-GANs~\citep{liu2017approximation,liu2018the} where the function family $\gF$ is a RKHS with fixed ``feature maps''. Different from these approaches, ours 
naturally allows the ``feature maps'' to be learned. Moreover, considering both restrictions allows us to bypass inner-loop optimization via closed-form solutions in certain cases (such as KL or $\chi^2$ divergences); this leads to our KL-WGAN approach which is easy to implement from existing WGAN implementations, and also have similar computational cost per iteration.

\subsection{
Reweighting of Generated Samples} 
The learned discriminators in GANs can further be used to perform reweighting over the generated samples~\citep{tao2018chi}; these include rejection sampling~\citep{azadi2018discriminator}, importance sampling~\citep{grover2019bias,tao2018chi}, and Markov chain monte carlo~\citep{turner2018metropolis}. These approaches can only be performed after training has finished, unlike our KL-WGAN case where discriminator-based reweighting are performed during training. 

Moreover, prior reweighting approaches assume that the discriminator learns to approximate some (fixed) function of the density ratio $\diff \pdata / \diff Q_\theta$, which does not apply directly to general IPM-based GAN objectives (such as WGAN); in KL-WGAN, we interpret the discriminator outputs as (un-normalized, regularized) log density ratios, introducing the density ratio interpretation to the IPM paradigm. We note that post-training discriminator-based reweighting can also be applied to our approach, and is orthogonal to our contributions; we leave this as future work.

\begin{table*}[th]
\centering
\caption{Negative Log-likelihood (NLL) and Maximum mean discrepancy (MMD, multiplied by $10^3$) results on six 2-d synthetic datasets. Lower is better. W denotes the original WGAN objective, and KL-W denotes the proposed KL-WGAN objective.}
\vskip 0.15in
\label{tab:synthetic}
\begin{tabular}{c|c|cccccc}
\toprule
Metric & GAN & MoG & Banana & Rings & Square & Cosine & Funnel \\\midrule
 \multirow{2}{*}{NLL}  &  W  & $2.65 \pm 0.00$  & $3.61 \pm 0.02$  & $\textbf{4.25} \pm 0.01$  & $3.73 \pm 0.01$  & $\textbf{3.98} \pm 0.00$  & $3.60 \pm 0.01$ \\
& KL-W & $\textbf{2.54} \pm 0.00$  & $\textbf{3.57} \pm 0.00$  & $\textbf{4.25} \pm 0.00$  & $\textbf{3.72} \pm 0.00$  & $4.00 \pm 0.01$  & $\textbf{3.57} \pm 0.00$ \\\midrule
\multirow{2}{*}{MMD} & W & $25.45 \pm 7.78$  & $3.33 \pm 0.59$  & $2.05 \pm 0.47$  & $2.42 \pm 0.24$  & $\textbf{1.24} \pm 0.40$  & $1.71 \pm 0.65$ \\
& KL-W & $\textbf{6.51} \pm 3.16$  & $\textbf{1.45} \pm 0.12$  & $\textbf{1.20} \pm 0.10$  & $\textbf{1.10} \pm 0.23$  & $1.33 \pm 0.23$  & $\textbf{1.08} \pm 0.23$ \\\bottomrule
\end{tabular}
\end{table*}

\begin{figure*}[th]
    \centering
\begin{subfigure}{0.7\textwidth}
    \centering
    \includegraphics[width=\textwidth]{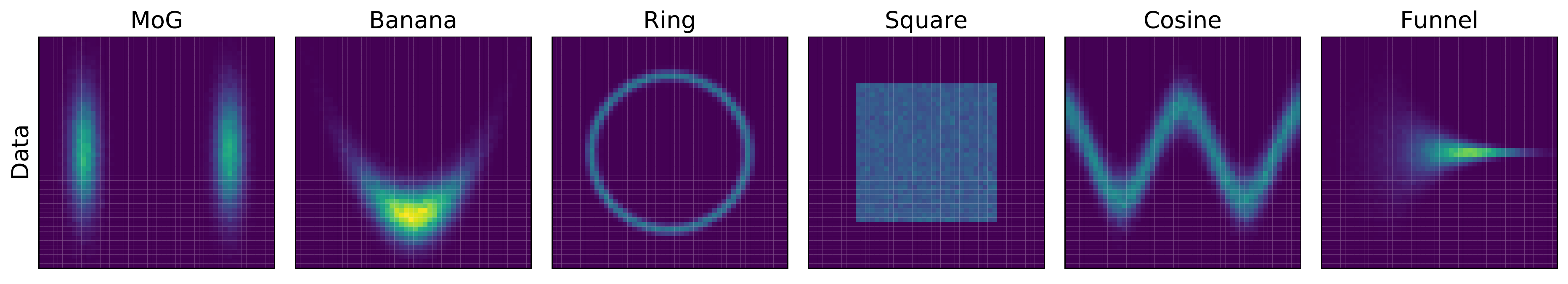}
\end{subfigure}
\begin{subfigure}{0.7\textwidth}
    \centering
    \includegraphics[width=\textwidth]{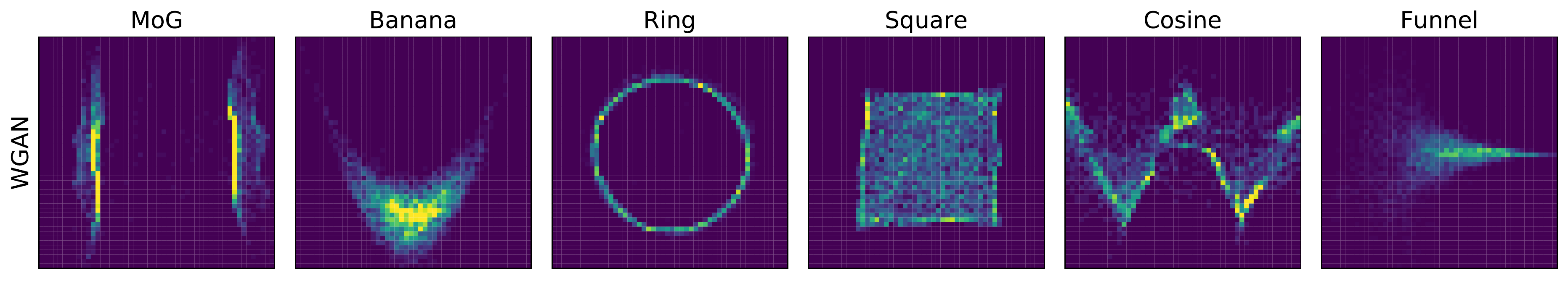}
\end{subfigure}
\begin{subfigure}{0.7\textwidth}
    \centering
    \includegraphics[width=\textwidth]{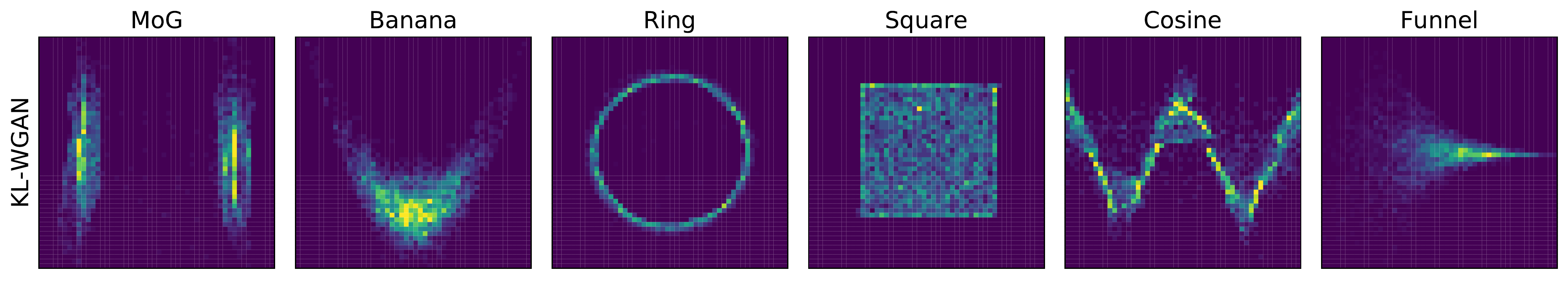}
\end{subfigure}
    \caption{Histograms of samples from the data distribution (top), WGAN (middle) and our KL-WGAN (bottom).}
    \label{fig:hist2d}
\end{figure*}

\begin{figure}
    \centering
\begin{subfigure}{0.43\textwidth}
    \centering
    \includegraphics[width=\textwidth]{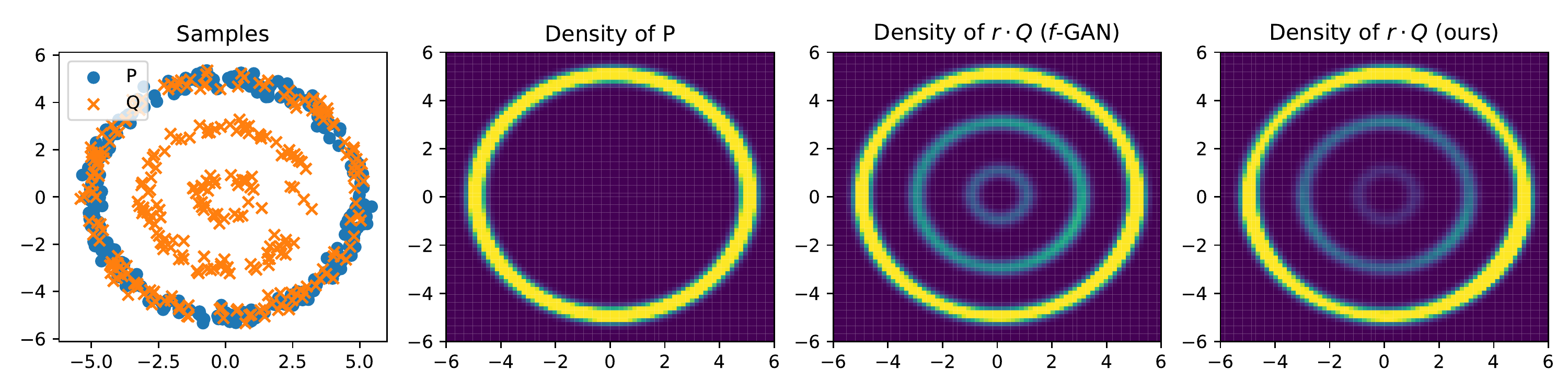}
\end{subfigure}
\begin{subfigure}{0.43\textwidth}
    \centering
    \includegraphics[width=\textwidth]{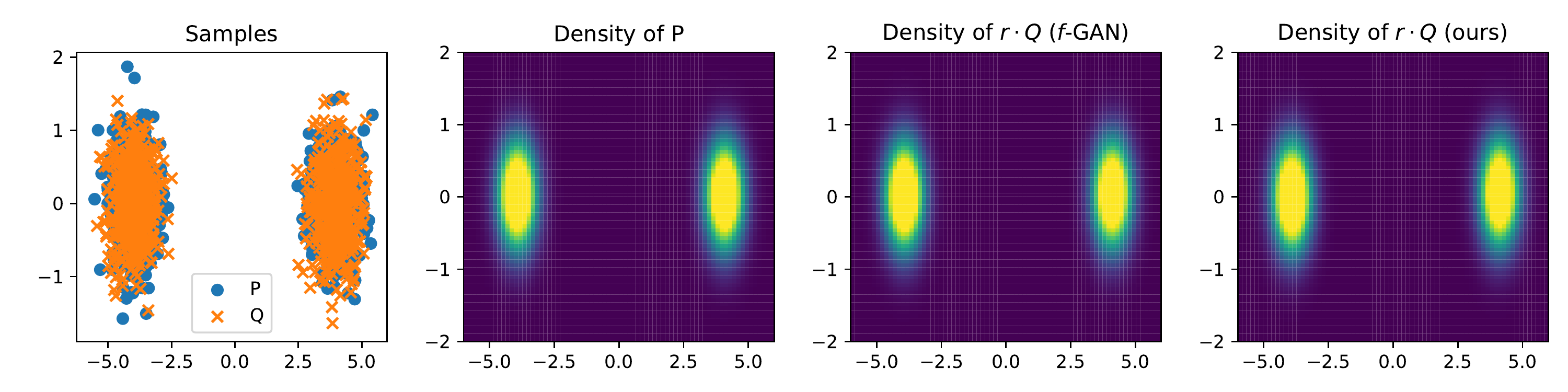}
\end{subfigure}
\begin{subfigure}{0.43\textwidth}
    \centering
    \includegraphics[width=\textwidth]{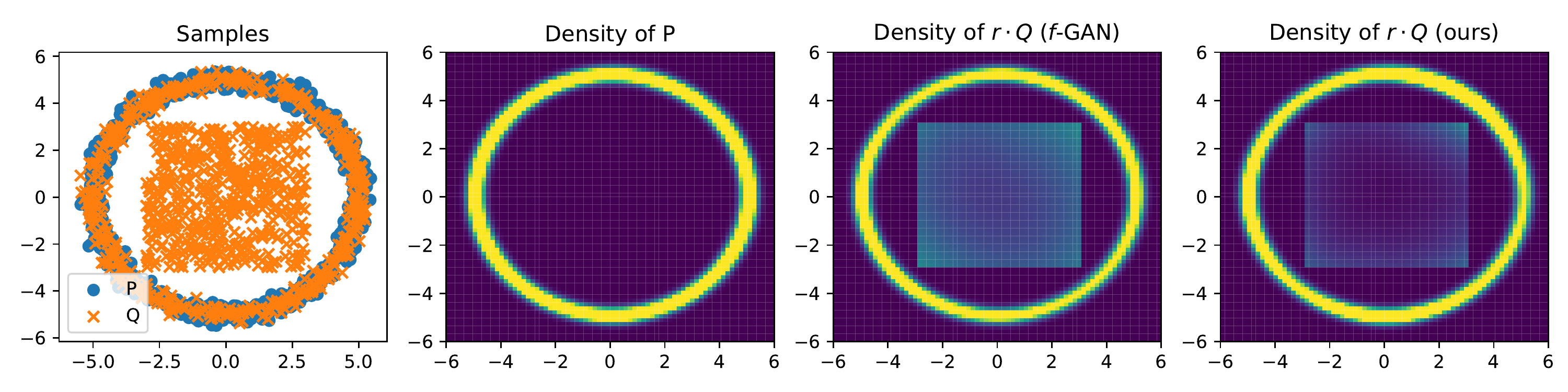}
\end{subfigure}
\begin{subfigure}{0.43\textwidth}
    \centering
    \includegraphics[width=\textwidth]{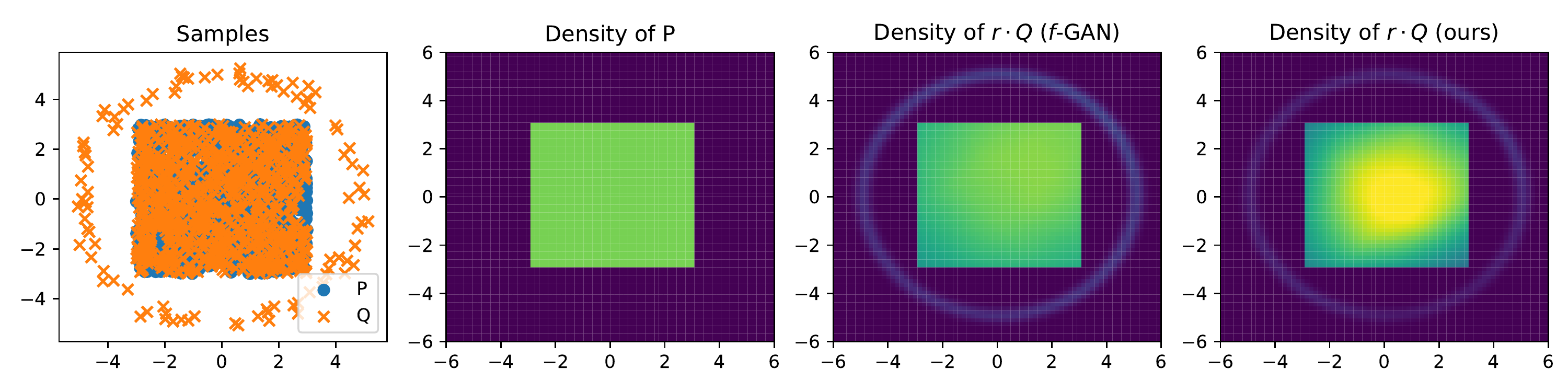}
\end{subfigure}
    \caption{Estimating density ratios. The first column contains the samples used for training, the second column is the ground truth density of $P$, the third and fourth  columns are the density of $Q$ times the estimated density ratios from original $f$-GAN (third column) and our KL-WGAN (fourth column).}
    \label{fig:density-ratio}
\end{figure}

\begin{table}[t]
\centering
\caption{Negative Log-likelihood (NLL, top two rows) and Maximum mean discrepancy (MMD, multiplied by $10^3$, bottom two rows) results on real-world datasets. Lower is better for both evaluation metrics. W denotes the original WGAN objective, and KL denotes the proposed KL-WGAN objective.}
\vskip 0.15in
\label{tab:real}
\begin{tabular}{c|ccc}
\toprule
 & RedWine & WhiteWine & Parkinsons \\\midrule
 W  & $14.55 \pm 0.04$  & $14.12 \pm 0.02$  & $20.24 \pm 0.08$   \\
 KL & $\textbf{14.41} \pm 0.03$  & $\textbf{14.08} \pm 0.02$  & $\textbf{20.16} \pm 0.05$ \\\midrule
 W & $2.61 \pm 0.37$  & $1.32 \pm 0.10$  & $1.30 \pm 0.09$  \\
KL & $\textbf{2.55} \pm 0.11$  & $\textbf{1.23} \pm 0.17$  & $\textbf{0.84} \pm 0.04$ \\\bottomrule
\end{tabular}
\end{table}

\begin{figure*}[th]
    \centering
    \includegraphics[width=0.85\textwidth]{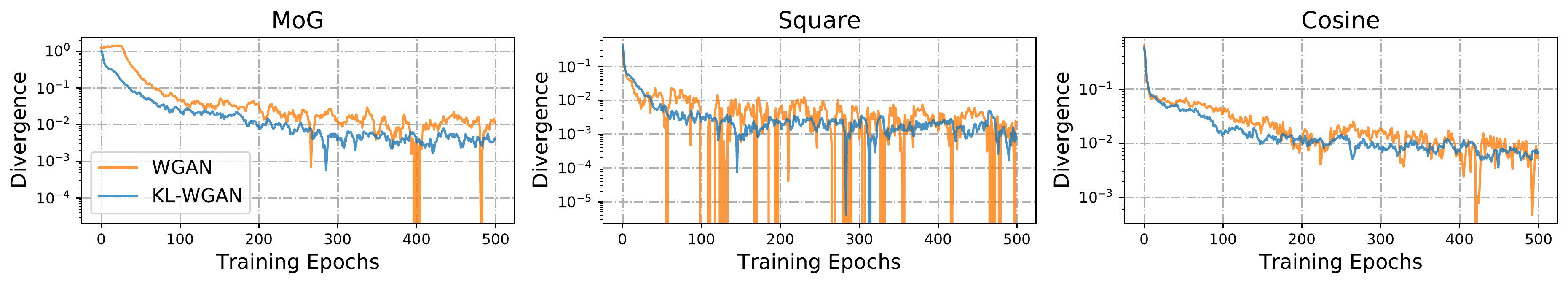}
    \caption{Estimated divergence with respect to training epochs (smoothed with a window of 10).}
    \label{fig:div_curve}
\end{figure*}

\section{Experiments}

We release code for our experiments (implemented in PyTorch) in \href{{https://github.com/ermongroup/f-wgan}}{https://github.com/ermongroup/f-wgan}.

\subsection{Synthetic and UCI Benchmark Datasets}

We first demonstrate the effectiveness of KL-WGAN on synthetic and UCI benchmark datasets~\citep{asuncion2007uci} considered in~\citep{wenliang2018learning}. The 2-d synthetic datasets include Mixture of Gaussians (MoG), Banana, Ring, Square, Cosine and Funnel; these datasets cover different modalities and geometries. We use RedWine, WhiteWine and Parkinsons from the UCI datasets. We use the same SNGAN~\citep{miyato2018spectral} arhictetures for WGAN and KL-WGANs, which uses spectral normalization to enforce Lipschitzness (detailed in Appendix~\ref{app:exps}).

After training, we draw 5,000 samples from the generator and then evaluate two metrics over a fixed validation set. One is the negative log-likelihood (NLL) of the validation samples on a kernel density estimator fitted over the generated samples; the other is the maximum mean discrepancy (MMD,~\citet{borgwardt2006integrating}) between the generated samples and validation samples. To ensure a fair comparison, we use identical kernel bandwidths for all cases.

\paragraph{Distribution modeling} We report the mean and standard error for the NLL and MMD results in Tables~\ref{tab:synthetic} and~\ref{tab:real} (with 5 random seeds in each case) for the synthetic datasets and UCI datasets respectively.
The results demonstrate that our KL-WGAN approach  outperforms its WGAN counterpart on all but the Cosine dataset. 
From the histograms of samples in Figure~\ref{fig:hist2d}, we can visually observe where our KL-WGAN performs significantly better than WGAN. For example, WGAN fails to place enough probability mass in the center of the Gaussians in MoG and fails to learn a proper square in Square, unlike our KL-WGAN approaches.

\paragraph{Density ratio estimation} We demonstrate that adding the constraint $r \in \Delta(Q)$ leads to effective density ratio estimators. We consider measuring the density ratio from synthetic datasets, and compare them with the original $f$-GAN with KL divergence. We evaluate the density ratio estimation quality by multiplying $\diff Q$ with the estimated density ratios, and compare that with the density of $P$; ideally the two quantities should be identical. We demonstrate empirical results in Figure~\ref{fig:density-ratio}, where we plot the samples used for training, the ground truth density of $P$ and the two estimates given by two methods.
In terms of estimating density ratios, our proposed approach is comparable to the $f$-GAN one.

\paragraph{Stability of critic objectives} For the MoG, Square and Cosine datasets, we further show the estimated divergences over a batch of 256 samples in Figure~\ref{fig:div_curve}, where WGAN uses $I_W$ and KL-WGAN uses the proposed $\gL_f^{\Delta(Q)}$. While both estimated divergences decrease over the course of training, our KL-WGAN divergence is more stable on all three cases. 
In addition, we evaluate the number of occurrences when a negative estimate of the divergences was produced for an epoch (which contradicts the fact that divergences should be non-negative); over 500 batches, WGAN has 46, 181 and 55 occurrences on MoG, Square and Cosine respectively, while KL-WGAN only has 29, 100 and 7 occurrences. This suggests that the proposed objective is easier to estimate and optimize, and is more stable across different iterations. %

\begin{table}
    \centering
    \caption{Inception and FID scores for CIFAR10 image generation. We list comparisons with results reported by %
    WGAN-GP~\citep{gulrajani2017improved}, Fisher GAN~\citep{mroueh2017fisher}, $\chi^2$ GAN~\citep{tao2018chi}, MoLM~\citep{ravuri2018learning}, SNGAN~\citep{miyato2018spectral}, NCSN~\citep{song2019generative}, BigGAN~\citep{brock2018large} and Sphere GAN~\citep{park2019sphere}. 
    (*) denotes our experiments with the PyTorch BigGAN implementation. %
    }
    \label{tab:cifar10}
    \vskip 0.15in
    \begin{tabular}{l|cc}
    \toprule
        Method &  Inception score &  FID score \\\midrule
        \multicolumn{3}{l}{\textbf{CIFAR10 Unconditional}} \\\midrule
        WGAN-GP & $7.86 \pm .07$ & - \\
        Fisher GAN & $7.90 \pm .05$ & - \\
        MoLM & $7.90 \pm .10$ & $18.9$ \\
        SNGAN & $8.22 \pm .05$ & 21.7 \\
        Sphere GAN & $8.39 \pm .08$ & $\textbf{17.1}$ \\
        NCSN & $\textbf{8.91}$ & 25.32 \\\midrule
        BigGAN* & $8.60 \pm .10$ & $16.38$ \\
        KL-BigGAN* & $\textbf{8.66} \pm .09$ & $\textbf{15.23}$ \\
        \midrule
        \multicolumn{3}{l}{\textbf{CIFAR10 Conditional}} \\\midrule
        Fisher GAN & $8.16 \pm .12$ & - \\
        WGAN-GP & $8.42 \pm .10$ & - \\
        $\chi^2$-GAN & $8.44 \pm .10$ & - \\
        SNGAN & $8.60 \pm .08$ & $17.5$ \\
        BigGAN & $\textbf{9.22}$ & $\textbf{14.73}$ \\
        \midrule
        BigGAN* & $9.08 \pm .11$ & $9.51$ \\
        KL-BigGAN* & $\textbf{9.20} \pm .09$ & $\textbf{9.17}$ \\\bottomrule
    \end{tabular}
\end{table}

\begin{table}
    \centering
    \caption{FID scores for CelebA image generation. The mean and standard deviation are obtained from 4 instances trained with different random seeds.}
    \vskip 0.15in
    \begin{tabular}{l|c|c}
    \toprule
        Method & Image Size & FID score \\\midrule
         BigGAN & \multirow{2}{*}{$64 \times 64$}& $18.07 \pm 0.47$ \\
        KL-BigGAN & &  $\textbf{17.70} \pm 0.32$\\
        \bottomrule
    \end{tabular}
    
    \label{tab:celeba}
\end{table}

\subsection{Image Generation}

We further evaluate our KL-WGAN's practical on image generation tasks on CIFAR10 and CelebA datasets.  
Our experiments are based on the BigGAN~\citep{brock2018large} PyTorch implementation\footnote{\hyperref[https://github.com/ajbrock/BigGAN-PyTorch]{https://github.com/ajbrock/BigGAN-PyTorch}}. We use a smaller network than the one reported in~\citet{brock2018large} (implemented on TensorFlow), using the default architecture in the PyTorch implementation.

We compare training a BigGAN network with its original objective and training same network with our proposed KL-WGAN algorithm, where we add steps~\ref{alg:kl-gan-fake-start} to~\ref{alg:kl-gan-fake-end} in Algorithm~\ref{alg:kl-wgan}. In addition, we also experimented with the original $f$-GAN with KL divergence; this failed to train properly due to numerical issues where exponents of very large critic values gives infinity values in the objective.

We report two common benchmarks for image generation, Inception scores~\citep{salimans2016improved} and Fréchet Inception Distance (FID)~\citep{heusel2017gans}~\footnote{Based on \href{https://github.com/mseitzer/pytorch-fid}{https://github.com/mseitzer/pytorch-fid}} in Table~\ref{tab:cifar10} (CIFAR10) and Table~\ref{tab:celeba} (CelebA). 
We do not report inception score on CelebA since the real dataset only has a score of less than 3, so the score is not very indicative of generation performance~\citep{heusel2017gans}. We show generated samples from the model in Appendix~\ref{app:samples}.

Despite the strong performance of BigGAN, our method is able to consistently achieve superior inception scores and FID scores consistently on all the datasets and across different random seeds. This demonstrates that the KL-WGAN algorithm is practically useful, and can serve as a viable drop-in replacement for the existing WGAN objective even on state-of-the-art GAN models, such as BigGAN.

\section{Conclusions}
In this paper, we introduce a generalization of $f$-GANs and WGANs based on optimizing a (regularized) objective over importance weighted samples. This perspective allows us to recover both $f$-GANs and WGANs when different sets to optimize for the importance weights are considered.
In addition, we show that this generalization leads to alternative practical objectives for training GANs and demonstrate its effectiveness on several different applications, such as distribution modeling, density ratio estimation and image generation. The proposed method only requires a small change in the original training algorithm and is easy to implement in practice. %

In future work, we are interested in considering other constraints that could lead to alternative objectives and/or inequalities and their practical performances. It would also be interesting to investigate the KL-WGAN approaches on high-dimensional density ratio estimation tasks such as off-policy policy evaluation, inverse reinforcement learning and contrastive representation learning.

\section*{Acknowledgements}
The authors would like to thank Lantao Yu, Yang Song, Abhishek Sinha, Yilun Xu and Shengjia Zhao for helpful discussions about the idea, proofreading a draft, and details about the image generation experiments. This research was supported by AFOSR (FA9550-19-1-0024), NSF (\#1651565, \#1522054, \#1733686), ONR, and FLI.

\bibliography{bib}
\bibliographystyle{icml2020}

\appendix

\newpage
\onecolumn

\section{Proofs}
\label{app:proofs}

\lfrbounded*
\begin{proof}
From Propositions~\ref{thm:lr-f}, and that $\gR \subseteq L^\infty_{\geq 0}(Q)$, we have:
\begin{align}
 \iftpq  = \inf_{r \in L^\infty_{\geq 0}(Q)} \ltrpq \leq \inf_{r \in \gR} \ltrpq = \LRtpq.
\end{align}
From Proposition~\ref{thm:lr-w} and that $\{\mathds{1}\} \subseteq \gR$, we have:
\begin{align}
 \LRtpq  = \inf_{r \in \gR} \ltrpq \leq \inf_{r \in \mathds{1}} \ltrpq = \LRtpq \leq \iwtrpq.
\end{align}
Combining the two inequalities completes the proof.
\end{proof}

\div*
\begin{proof}
From Proposition~\ref{thm:lr-f}, we have the following upper bound for $D_{f, \gR}(P \Vert Q)$:
\begin{align}
  & \ \sup_{T \in \gF} \inf_{r \in \gR} \bb{E}_{P}[f(r)] + \bb{E}_P[T] - \bb{E}_Q[r \cdot T] \\
  \leq & \ \sup_{T \in \gF} \inf_{r \in \{\mathds{1}\}} \bb{E}_{P}[f(r)] + \bb{E}_P[T] - \bb{E}_Q[r \cdot T] \nonumber \\
  = & \ \sup_{T \in \gF} \bb{E}_P[T] - \bb{E}_Q[T] = \ipm_{\gF}(P, Q), \nonumber
\end{align}
We also have the following lower bound for $D_{f, \gR}(P \Vert Q)$:
\begin{align}
  & \ \sup_{T \in \gF} \inf_{r \in \gR} \bb{E}_{P}[f(r)] + \bb{E}_P[T] - \bb{E}_Q[r \cdot T] \\
  \geq & \ \sup_{T \in \gF} \inf_{r \in L_{\geq 0}^\infty(Q)} \bb{E}_{P}[f(r)] + \bb{E}_P[T] - \bb{E}_Q[r \cdot T] \nonumber \\
  = & \ \sup_{T \in \gF} \bb{E}_P[T] - \bb{E}_Q[f_{*}(T)] = D_f(P \Vert Q). \nonumber
\end{align}
Therefore, $D_{f, \gR}(P \Vert Q)$ is bounded between $D_f(P \Vert Q)$ and $\ipm_{\gF}(P, Q)$ and thus it is a valid divergence over $\gP(\gX)$.
\end{proof}

\klgan*
\begin{proof}
Consider the following Lagrangian:
\begin{align}
  h(r, \lambda) := \E_{Q}[f(r)] - \bb{E}_{Q}[r \cdot T] + \lambda (\E_{Q}[r] - 1)
\end{align}
where $\lambda \in \R$ and we formalize the constraint $r \in \Delta(r)$ with $\E_{Q}[r] - 1 = 0$. Taking the functional derivative $\partial h / \partial r$ and setting it to zero, we have:
\begin{align}
   & \ f'(r) \diff Q - T \diff Q + \lambda \\
   =& \ (\log r + 1) \diff Q - T \diff Q + \lambda = 0, \nonumber
\end{align}
so $r = \exp(T - (\lambda + 1))$. We can then apply the constraint $\E_{Q}[r] = 1$, where we solve $\lambda + 1 = \bb{E}_Q[e^{T}]$, and consequently the optimal $r = e^{T} / \bb{E}_Q[e^{T}] \in \Delta(Q)$.
\end{proof}

\section{Example KL-WGAN Implementation in PyTorch}
\label{app:klwgan-impl}
{\small
\begin{verbatim}
def get_kl_ratio(v):
    vn = torch.logsumexp(v.view(-1), dim=0) - torch.log(torch.tensor(v.size(0)).float())
    return torch.exp(v - vn)

def loss_kl_dis(dis_fake, dis_real, temp=1.0):
    """
    Critic loss for KL-WGAN.
    dis_fake, dis_real are the critic outputs for generated samples and real samples.
    temp is a hyperparameter that scales down the critic outputs.
    We use the hinge loss from BigGAN PyTorch implementation.
    """
    loss_real = torch.mean(F.relu(1. - dis_real))
    dis_fake_ratio = get_kl_ratio(dis_fake / temp)
    dis_fake = dis_fake * dis_fake_ratio
    loss_fake = torch.mean(F.relu(1. + dis_fake))
    return loss_real, loss_fake

def loss_kl_gen(dis_fake, temp=1.0):
    """
    Generator loss for KL-WGAN.
    dis_fake is the critic outputs for generated samples.
    temp is a hyperparameter that scales down the critic outputs.
    We use the hinge loss from BigGAN PyTorch implementation.
    """
    dis_fake_ratio = get_kl_ratio(dis_fake / temp)
    dis_fake = dis_fake * dis_fake_ratio
    loss = -torch.mean(dis_fake)
    return loss
\end{verbatim}}

\section{Argument about $\chi^2$-Divergences}
\label{app:chi-square}
We present a similar argument to Theorem~\ref{thm:kl-gan} to $\chi^2$-divergences, where $f(u) = (u - 1)^2$.

\begin{theorem}\label{thm:chi2-gan}
Let $f(u) = (u-1)^2$ and $\gF$ is a set of real-valued bounded measurable functions on $\gX$. For any fixed choice of $P, Q$, and $T \in \gF$ such that $T \geq 0, T - \bb{E}[T] + 2 \geq 0$, we have
\begin{gather}
    \argmin_{r \in \Delta(Q)} \E_{Q}[f(r)] + \bb{E}_P[T] - \bb{E}_{Q_r}[T] = \frac{T - \E_{Q}[T] + 2}{2} \nonumber
\end{gather}
\end{theorem}

\begin{proof}
Consider the following Lagrangian:
\begin{align}
  h(r, \lambda) := \E_{Q}[f(r)] - \bb{E}_{Q}[r \cdot T] + \lambda (\E_{Q}[r] - 1)
\end{align}
where $\lambda \in \R$ and we formalize the constraint $r \in \Delta(r)$ with $\E_{Q}[r] - 1 = 0$. Taking the functional derivative $\partial h / \partial r$ and setting it to zero, we have:
\begin{align}
   & \ f'(r) \diff Q - T \diff Q + \lambda \\
   =& \ 2r \diff Q - T \diff Q + \lambda = 0, \nonumber
\end{align}
so $r = (T - \lambda) / 2$. We can then apply the constraint $\E_{Q}[r] = 1$, where we solve $\lambda = \bb{E}_Q[T] - 2$, and consequently the optimal $r = (T - \E_{Q}[T] + 2) / 2 \in \Delta(Q)$.
\end{proof}

In practice, when the constraint $T - \bb{E}_Q[T] + 2 \geq 0$ is not true, then one could increase the values when $T$ is small, using
\begin{align}
    \hat{T} = \max(T, c) + b
\end{align}
where $b, c$ are some constants that satisfies $\hat{T(\vx)} - \bb{E}_Q[\hat{T}] + 2 \geq 0$ for all $\vx \in \gX$. Similar to the KL case, we encourage higher weights to be assigned to higher quality samples.

If we plug in this optimal $r$, we obtain the following objective:
\begin{align}
    \E_P[T] - \E_Q[T] + \frac{1}{4} \E_Q[T^2] + \frac{1}{4} (\E_Q[T])^2 = \E_P[T] - \E_Q[T] - \frac{\Var_Q[T]}{4}.
\end{align}

Let us now consider $P = \pdata$, $Q = \frac{\pdata + G_\theta}{2}$, then the $f$-divergence corresponding to $f(u) = (u-1)^2$:
\begin{align}
    D_f(P \Vert Q) = \int_\gX \frac{(P(\vx) - Q(\vx))^2}{\frac{P(\vx) + Q(\vx)}{2}} \diff \vx,
\end{align}
is the squared $\chi^2$-distance between $P$ and $Q$. So the objective becomes:
\begin{align}
  \min_\theta \max_\phi  \E_{\pdata}[D_\theta] - \E_{G_\theta}[D_\phi] - \Var_{M_\theta}[D_\phi],
\end{align}
where $M_\theta = (\pdata + G_\theta) / 2$ and we replace $T/2$ with $D_\phi$. In comparison, the $\chi^2$-GAN objective~\citep{tao2018chi} for $\theta$ is:
\begin{align}
    \frac{(\E_{\pdata}[D_\theta] - \E_{G_\theta}[D_\phi])^2}{\Var_{M_\theta}[D_\phi]}.
\end{align}
They do not exactly minimize $\chi^2$-divergence, or a squared $\chi^2$-divergence, but a normalized version of the 4-th power of it, hence the square term over $\E_{\pdata}[D_\theta] - \E_{G_\theta}[D_\phi]$.

\section{Additional Experimental Details}
\label{app:exps}
For 2d experiments, we consider the WGAN and KL-WGAN objectives with the same architecture and training procedure. Specifically, our generator is a 2 layer MLP with 100 neurons and LeakyReLU activations on each hidden layer, with a latent code dimension of 2; our discriminator is a 2 layer MLP with 100 neurons and LeakyReLU activations on each hidden layer. We use spectral normalization~\citep{miyato2018spectral} over the weights for the generators and consider the hinge loss in~\citep{miyato2018spectral}. Each dataset contains 5,000 samples from the distribution, over which we train both models for 500 epochs with RMSProp (learning rate 0.2). The procedure for tabular experiments is identical except that we consider networks with 300 neurons in each hidden layer with a latent code dimension of 10. Dataset code is contained in {\hyperref[https://github.com/kevin-w-li/deep-kexpfam]{https://github.com/kevin-w-li/deep-kexpfam}}.

\section{Samples}
\label{app:samples}
We show uncurated samples from BigGAN trained with WGAN and KL-WGAN loss in Figures~\ref{fig:biggan-samples} and~\ref{fig:klbiggan-samples}.

\begin{figure}[H]
    \centering
    \begin{subfigure}{0.4\textwidth}
    \includegraphics[width=\textwidth]{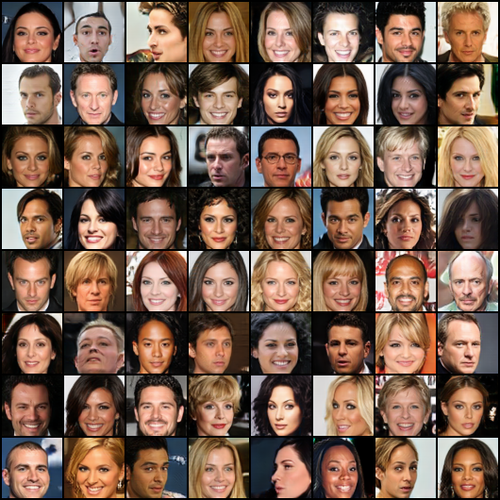}
    \caption{CelebA 64x64 samples trained with WGAN.}
    \label{fig:biggan-samples}
    \end{subfigure}
    ~
    \begin{subfigure}{0.4\textwidth}
    \includegraphics[width=\textwidth]{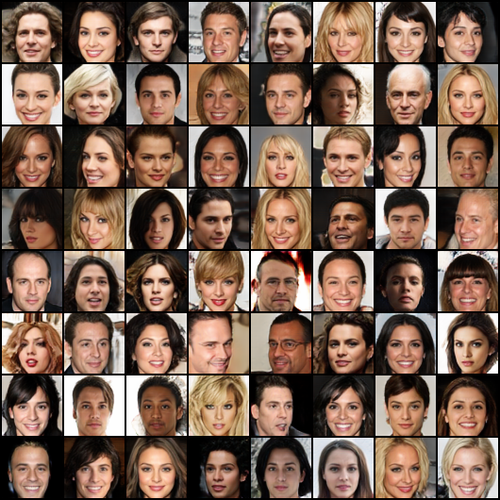}
    \caption{CelebA 64x64 Samples trained with KL-WGAN.}
    \label{fig:klbiggan-samples}
    \end{subfigure}
\end{figure}

\begin{figure}[H]
    \centering
    \begin{subfigure}{0.55\textwidth}
    \includegraphics[width=\textwidth]{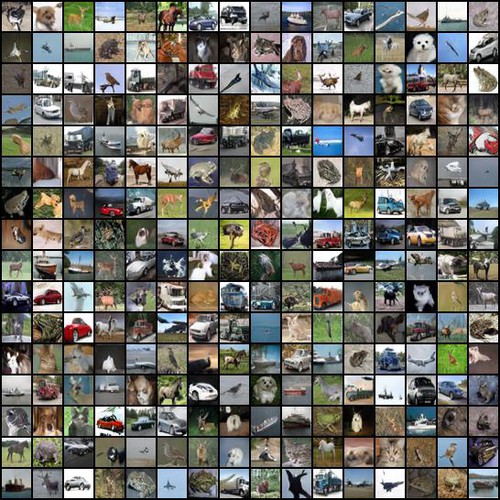}
    \caption{CIFAR samples trained with WGAN.}
    \label{fig:biggan-samples}
    \end{subfigure}
    ~
    \begin{subfigure}{0.55\textwidth}
    \includegraphics[width=\textwidth]{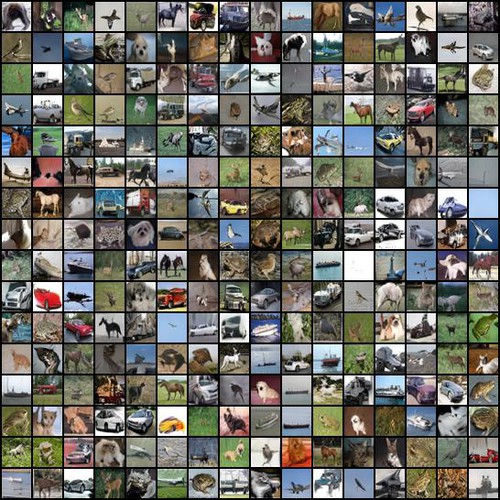}
    \caption{CIFAR samples trained with KL-WGAN.}
    \label{fig:klbiggan-samples}
    \end{subfigure}
\end{figure}

\end{document}